\documentclass[12pt]{article}
\usepackage{epsfig, graphics, subfigure}
\usepackage{amsfonts}
\usepackage{natbib}
\usepackage{amsmath,amsthm,amssymb,enumerate,mathrsfs}
\usepackage{eurosym}
\usepackage{titling}
\usepackage{appendix}
\usepackage{setspace}
\usepackage{float}

\newtheoremstyle{colon}%
{}
{}
{\itshape}
{}
{\bfseries}
{:}
{ }
{}

\makeatletter
\let\@fnsymbol\@arabic
\makeatother
\usepackage{mathtools, bm}
\usepackage{amssymb, bm}
\usepackage{bbm}
\usepackage[latin1]{inputenc}

\def\bkR{{\rm I\kern-.17em R}}

\def \1n{1\hskip -3pt \mbox{N}}

\theoremstyle{colon}

\newtheorem{prop}{Proposition}

\newtheorem{Ass}{Assumption}
\newtheorem{coro}{Corollary}

\theoremstyle{definition}

\newtheorem{rem}{Remark}

\newtheorem{exm}{Example}
\usepackage{setspace}
\begin{titlepage}

\title{A simple estimator of the correlation kernel matrix of a determinantal point process}
\author{Christian Gouri\'eroux\thanks{University of Toronto, Toulouse School of Economics and CREST. Email: gouriero@ensae.fr} \, and Yang Lu\thanks{Department of Mathematics and Statistics, Concordia University, Montreal, Canada. Email: yang.lu@concordia.ca}}

\end{titlepage}
\begin{document}

\maketitle

\textbf{Abstract}: The Determinantal Point Process (DPP) is a parameterized model for multivariate binary variables, characterized by a correlation kernel matrix. This paper proposes a closed form estimator of this kernel, which is particularly easy to implement and can also be used as a starting value of learning algorithms for maximum likelihood estimation. We prove the consistency and asymptotic normality of our estimator,  as well as its large deviation properties.  


\vspace*{0.5em}
\textbf{Acknowledgment:} Gouri\'eroux gratefully acknowledges financial support of the ACPR chair ``Regulation and Systemic Risks" and the ANR project: ``From Machine Learning to Structural Econometrics with Discrete Data".  Lu thanks NSERC (grant RGPIN- 2021-04144) for financial support.  

\section{Introduction}
 
Determinantal Point Process (DPP) is a flexible family of distributions for random sets defined on the finite state space $\{1,...,d\}$, or equivalently for multivariate binary variables.  This family is parameterized by either the L-ensemble kernel $\Sigma$, which is symmetric positive definite (SPD), or the correlation kernel matrix $K$, which is SPD, with eigenvalues lying strictly between 0 and 1. The literature has considered the maximum likelihood estimation (MLE) of $\Sigma$ and $K$ or its algorithmic analogues  \citep{affandi2014learning, brunel2017maximum, brunel2017rates}, but it has since been shown that $i)$ the likelihood function has at least $2^{d-1}$ global maxima due to the identifiability property of the DPP model; $ii)$ the likelihood function is highly non convex and hence suffers from issues like local minima and saddlepoints, whose numbers increase exponentially in dimension $d$ \citep{brunel2017rates, friedman2024likelihood}; $iii)$ the Hessian matrix at the true parameter value is close to singularity, especially at high dimensions \citep{brunel2017rates}; $iv)$ in computational complexity,  finding the MLE is considered a difficult problem \citep{kulesza2012learning, grigorescu2022hardness}.  The numerical performance of the MLE has been studied by \cite{hu2023asymptotic},  who report that both the deterministic Newton-Ralphson and the stochastic gradient descent optimization algorithms for the maximum likelihood estimation of $\Sigma$ can fail to converge, even in low dimensions such as $d=3$. 

Another approach is to use moment-based estimators or algorithms.  In particular,  \cite{urschel2017learning} propose the first such algorithm, which can be applied to quite general cases, including when many entries of $K$ are zero. 
However,  this method also has several downsides.  First,  the method relies on quite complicated graph theory and the number of moments needed depends on the unknown value of the kernel. Moreover, the consistency and asymptotic normality of their estimator is not known.  Our paper shows that under some mild conditions,  such as one of the row of the kernel matrix only contains nonzero terms\footnote{That is to say,  one binary variable is correlated with all other binary variables. },  an alternative (method of moment) estimation of the kernel correlation matrix exists,  and is particularly easy to implement.  Moreover,  we were also able to prove standard asymptotic properties of this estimator,  such as the strong consistency,  the asymptotic normality,  as well as large deviation properties.  It can be used either as a standalone estimator, or a starting point of the maximum likelihood estimation algorithms.  

 The plan of the paper is the following. In Section 2, we introduce the DPP model and its parametrization by the correlation kernel matrix $K$. The identification issues are discussed in Section 3 for the unconstrained DPP model. Indeed, the distribution depends on the kernel by means of all its principal minors and these principal minors are not sufficient to identify the kernel itself. We also explain why the principal minors of order non larger than 3 are sufficient to recover all other principal minors. In Section 4, the discussion of identification is extended to constrained DPP models. Section 5 concerns statistical inference. We introduce a closed form estimator of the kernel that converges almost surely to a pseudo-true kernel in the identified set. This almost sure convergence is a key point to derive the asymptotic normality of the estimated identified set. Section 6 compares our approach with the alternative estimation approach of \cite{urschel2017learning}, and derives its large deviation properties. Section 7 concludes.  One technical proof is relegated to the Appendix.

\section{The DPP model}
Let us denote by $S$ the random subset of the state space $\{1,...,d\}$, which can be equivalently represented by a $d$-dimensional binary random vector $(X_1,...,X_d)$ defined through:
 $$S = \{1 \leq j \leq d, \text{ such that } X_{j}=1 \}.$$
We say that $S$ or $(X_1,...,X_d)$ follows a DPP, if its distribution is given by: 
\begin{equation}
\label{pmf}
p(s) = \mathbb{P}[S=s] =\mathbb{P}[X_j=1, \forall j \in s]= \frac{ \det \Sigma_s }{\det (I_d+\Sigma)}, \qquad \forall s \subseteq \{1,\ldots,d\},
\end{equation}
where $\Sigma$ is SPD,  and is called the L-ensemble kernel matrix or the exclusion kernel matrix of the DPP,  $\Sigma_s=[\Sigma_{ij}]_{i,j \in s}$ is the principal submatrix with index $s$, obtained by retaining only rows and columns whose indices belong to $s$, and $\det \Sigma_s,  s\subseteq \{1,...,d\}$, the set of principal minors.  By convention,  $\det \Sigma_{\emptyset}=1.$

Alternatively, the DPP can also be characterized by the cumulative distribution function of the random set variable $S$:
\begin{equation}
\label{cmf}
\mathbb{P} [s \subseteq S] = \det K_s, \qquad \forall s \subseteq \{1,\ldots,d\},
\end{equation}
 where {$K_s$ is the corresponding submatrix of the so-called \textit{correlation kernel} $K$ (henceforth kernel $K$). It has been shown in \cite{rising2013advances} that the correlation kernel $K$ is also SPD and is in a one-to-one relationship with $\Sigma$ through: 
\begin{equation}
\label{121}
\Sigma = K (I_d-K)^{-1}\qquad \Longleftrightarrow \qquad K=I_d-(I_d+ \Sigma)^{-1}=\Sigma (I_d+ \Sigma)^{-1}.
\end{equation}
In particular, all eigenvalues of $K$ lie strictly between 0 and 1.  Moreover,  the elementary probability \eqref{pmf} can be alternatively expressed through $K$ \cite[eq.(147)]{kulesza2012learning}:
$$
p(s)=\det(I_s K + I_{\bar{s}} (I_d -K)),
$$
 where $I_s$ is the diagonal matrix with ones in the diagonal positions corresponding to elements of $s$ and 0,  otherwise.  Thus one can either estimate the L-ensemble $\Sigma$  \citep{urschel2017learning},  or the correlation kernel matrix $K$ parametrization of the DPP \citep{affandi2014learning, urschel2017learning}.  
  In this paper, we take the second approach, which has the extra advantage that most of the simulation algorithms for DPP rely on $K$, rather than $\Sigma$ [see e.g. \cite{launay2020exact} and references therein]. 

\section{Identification in the unconstrained DPP model}
Let us assume that the DPP model is well-specified with true correlation kernel $K_0$. This model has the following identifiability property:
	\begin{prop}[\cite{rising2015efficient}]
		\begin{itemize}
			\item An SPD matrix $K$ and the true kernel $K_0$ are observationally equivalent, i.e. 
            {
            have the same principal minors}
            if and only if :
						\begin{equation}
\label{Dsimilar}
K= D K_0 D,
\end{equation}
						for some diagonal matrix $D$ whose diagonal elements all belong to $\{-1, 1\}$.  In other words,  we have :
						\begin{equation} 
							\label{expression}
							K_{i,j}=K_{0,i,j}d_{ii}d_{jj},
						\end{equation}
					 for any $i, j=1,...,d$.  {
                        In this case, we say that $K$ and $K_0$ are $D$-similar.} 
			\item 		If at least one row of the true parameter $K_0$ contains nonzero elements only, then  there are $2^{d-1}$   matrices $K$ that are D-similar with $K_0$.   
		\end{itemize}
        \end{prop}

        Proposition 1 can be interpreted as follows.  Let us assume i.i.d. observations $S_t,  t=1,...,T$, of the random set, where sample size $T$ is large.  Since we can derive consistent estimators of the distribution of $S$, this DPP distribution is identifiable,  as well as the set of principal minors $(\det K_s, s \subseteq \{1,...,d\})$ of the correlation kernel $K$, or the principal minors $(\det \Sigma_s, s \subseteq \{1,...,d\})$ of the L-ensemble $\Sigma$.  Proposition 1 shows that matrices $\Sigma$ and $K$ themselves are not identifiable and characterizes the identified set $IS_0$:
        $$
        IS_0=\{K, \text{ that are D-similar to } K_0  \},
        $$
        which is not reduced to the singleton $\{K_0\}$ in general.

            Since $D K_0 D=(-D)K_0 (-D)$, the number of observationally equivalent matrices is at most $2^{d-1}$, which explains the second part of Proposition 1.  We also make the following assumption:
        \begin{Ass}
       All elements of $K_0$ are nonzero. 
        \end{Ass}

        We can also check that if, say, the first row of $K_0$ has only nonzero elements, then, by eq.\eqref{expression}, for any diagonal matrix $D$ with $\pm 1$ elements on the diagonal, the first row of matrix $DK_0 D$ also satisfies Assumption 1.

 One important corollary of Proposition 1 is:
 \begin{coro}
Under Assumption 1, there is no loss of generality to assume that the first row of $K$ is positive,  that is,  $K_{1i} >0$ for each $i=2,....,d$.  
 \end{coro} 
 \begin{proof}
For a given diagonal matrix $D$ with $\pm 1$ on the diagonal, let us denote by $s_D$ the subset of indices $i$ such that $d_{ii}<0$.  Then the operation $K_0 \rightarrow D K_0 D$ has the effect of  changing first the sign of the rows indexed by $s_D$ of $K$, and then changing the sign of the columns indexed by $s_D$.  As a consequence,  by choosing an appropriate $D_0^*$,  we can get a matrix $K_0^*=D_0^* K_0D_0^*$, such that its first row is positive.  
  
 \end{proof}

Thus the transformation $K_0^*=D_0^* K_0D_0^*$, where $D_0^*$ is constructed from $K_0$, allows to replace the true kernel $K_0 \in IS_0$ by another element $K_0^* \in IS_0$, called a pseudo-true kernel.  From now on, we will assume that: 
\begin{Ass}
The first row of the pseudo-true kernel $K^*_0$ is positive.  
\end{Ass} 

This condition is an identification restriction.  That is, initially,  $K_0$ is only identified up to D-similarity,  but with Assumption 2,  it becomes possible to identify in a unique way $K^*_0$.  It will be shown in the estimation section that this is a huge advantage compared to existing estimation methods of the DPP, most of which have to deal with the fact that the objective function they maximize, such as the likelihood function, have at least $2^d$ global maxima.  




 We have mentioned that the set of principal minors $(\det K_s)$ or  $(\det \Sigma_s)$ are two equivalent parametrization of the DPP model.  However, they are not of minimal dimensions.  The following result allows us to focus on the principal minors of lower orders only as a system of generators.    

        \begin{prop}[\cite{stouffer1924independence}, Theorem 1]Under Assumption 1,  the principal minors of a square matrix $K$ are uniquely determined by the following $d^2-d+1$ principal minors:
        \begin{itemize}
\item the principal minors of order 1 and 2, that are $K_{ii}, i=1,...,d$, and $K_{ii}K_{jj}-K^2_{ij},  i \neq j=1,...,d$,
\item and the principal minors of order 3 of the type:
\begin{equation}
\det K_{\{1,i,j\}}= K_{11} K_{ii} K_{jj} +2 K_{1i} K_{ij} K_{1j}- K_{11} K^2_{ij}- K_{ii} K^2_{1j}- K_{jj} K^2_{1i},
\label{order3}
\end{equation} with $1 < i \neq j \leq d$.  
\end{itemize}
        \end{prop}
    
     As an illustration,  in the special case where $K$ is symmetric, by using the cofactor expansion along the last row,  one gets the principal minor of order 4:
    \begin{align*}
    	\det K_{\{1,2,3,4\}}  &= K_{44}\det  K_{\{1,2,3\}}- K_{34}^2\det K_{\{1,2\}}-K_{24}^2\det K_{\{1,3\}} + K_{14}^2 \det K_{\{2,3\}}   \\
    	& \qquad -2 K_{33} K_{12}K_{24}K_{14}- 2K_{22} K_{13}K_{14}K_{34}+2 K_{11} K_{23}K_{34}K_{23} \\
   & 2 K_{14} K_{13} K_{23} K_{24}+2 K_{12} K_{23} K_{34} K_{41}-2 K_{12} K_{13} K_{24} K_{34}.
    \end{align*}
In the  second line, we can apply eq.\eqref{order3} and express $K_{12}K_{24}K_{14}$,  $K_{13}K_{14}K_{34}$ and $K_{23}K_{34}K_{23}$ using $\det K_{\{1,2,4\}}$, $\det K_{\{1,3,4\}}$,  $\det K_{\{2,3,4\}}$, as well as the principal minors of order 1 and 2.  On the other hand,  the three terms in the third line can be expressed as:
    \begin{align*}
 K_{14} K_{13} K_{23} K_{24}&= \frac{ (K_{23} K_{13} K_{12} )(K_{14} K_{24}K_{12})}{K^2_{12}}, \\
 K_{12} K_{23} K_{34} K_{41}&= \frac{ ( K_{12} K_{23}  K_{13} )(K_{34} K_{41} K_{13} )}{K^2_{13}}, \\
 K_{12} K_{13} K_{24} K_{34}&= \frac{ ( K_{12} K_{24}  K_{14} )(K_{13} K_{34} K_{14} )}{K^2_{14}}, 
    \end{align*}
    where both the denominators and the numerators are functions of the principal minors of order at most 3. 

Proposition 2 is the key identification restriction of our paper.  As a comparison,  \cite{urschel2017learning}'s identification relies heavily on graph theory\footnote{See also \cite{brunel2018learning, brunel2024recovering} for a similar identification approach for non-symmetric DPP models. }. As \cite{urschel2017learning} put it (last paragraph, section 2.1),  ``\textit{for any $K$, there exists (a minimal order) $\ell_0$ depending only on the graph induced by $K$, such that $K$ can be recovered up to a D-similarity with only the knowledge of its principal minors of size at most $\ell.$ }" Most of the results concerning their identification and estimation concern this unknown  order $\ell_0$.  Proposition 2 simply says that under Assumption 1, this order is known: in fact $\ell_0=3$ and hence does not need to be estimated.

Proposition 2 is applicable to \textit{any} symmetric matrix, without $K^*_0$ being the correlation kernel of a DPP. In the following Proposition 3, we prove this result in the case where $K^*_0$ is the correlation kernel of a DPP, under both Assumptions 1 and 2. Indeed, in this case, the proof becomes constructive, in the sense that it provides a closed form expression of $K_0^*$ in terms of the true DPP distribution.  More precisely, we will explain how to derive all the elements of $K^*_0$ from  the knowledge of the marginal, pairwise and threewise probabilities of the binary variables $X_i, i=1,...,d$. For expository purpose, we omit below the ``$0$" index for the true distribution and kernel, and the ``$^*$" for the (pseudo) kernel, that is we use $K_{ij}$ instead of $K^*_{0,ij}$.
\begin{prop}
	Under Assumptions 1 and 2, the pseudo true value $K^*_0$ is identifiable through: 
	\begin{align}
		\label{identifymargin}
		K_{ii}&=\mathbb{P}[X_i=1]=\mathbb{E}[X_i], \qquad \forall i=1,...,d,\\
		\label{identifypairwise}
		\mid K_{ij} \mid&= \sqrt{ K_{ii} K_{jj}- \mathbb{P}[X_i=X_j=1]}=\sqrt{-Cov(X_i, X_j)}, \qquad  i <j,\\
		\label{identifythreewise}
		sgn(K_{ij})&= sgn \Big( \mathbb{P}[X_1=1, X_i=1, X_j=1]-K_{11}K_{ii}K_{jj}  \nonumber \\
		&\qquad+ K_{11} K^2_{ij}+ K_{ii} K^2_{1j}+K_{jj} K^2_{1j} \Big), \qquad  1<i <j \leq d,
	\end{align}
where $-Cov(X_i,X_j)$ is positive under Assumption 2, and the $sgn(\cdot)$ function is defined as $$sgn(x)=
\begin{cases}
1 & \text{ if } x>0 \\
-1 & \text{ if } x<0 \\
0& \text{ if } x=0
\end{cases}.$$  
Note that under Assumption 2, $sgn(K_{ij})$ can only be $\pm 1$, but not zero.    
\end{prop}

Clearly, Proposition 3 implies Proposition 2 in the case where $K$ is a kernel satisfying Assumptions 1 and 2.  

\begin{proof}
 
Eqs.\eqref{identifymargin} and \eqref{identifypairwise} are immediate consequences of eq.\eqref{cmf} by focusing on principal minors of order 1 and 2, respectively.  For \eqref{identifythreewise}, remark that:
$$
\mathbb{P}[X_1=X_i=X_j=1]= \det K_{\{1,i,j\}},
$$
whose expression is given in equation \eqref{order3}.  Thus we get:
$$
K_{ij}= \frac{1}{K_{1i}K_{1j}}\Big(\mathbb{P}[X_1=1, X_i=1, X_j=1] -K_{11}K_{ii}K_{jj}+K_{11} K^2_{ij}+ K_{ii} K^2_{1j}+K_{jj} K^2_{1j}\Big).
$$
Since $K_{1i}$, $K_{1j}$ are both positive by Assumption 2, we get eq.\eqref{identifythreewise}.

\end{proof} 
Proposition 3 is the key result of the paper.  Not only it will provide, in section 5, a closed form estimator of $K^*$ from observations of random sets,  but more generally,  it provides a much simpler algorithm for the Principal Minor Assignment problem for symmetric matrices,  compared to the existing approaches by \cite{griffin2006principal} and \cite{rising2015efficient}.

\section{Identification in the constrained DPP model}
The unconstrained DPP model depends on the $d(d+1)/2$ linearly independent parameters of the kernel matrix. Thus,  in some applications where $d$ is very large, we can encounter the curse of dimensionality. This suggests the introduction of constrained DPP model:
$$
K=K(\theta), \qquad \theta \in \Theta,
$$
where $K(\cdot)$ is a function and the dimension of the parameter $\theta$ is much smaller than $d(d+1)/2$. 

Similar constrained (discrete) DPP models have been considered by \cite{gartrell2017low, gartrell2019learning,  yu2025modeling},  and are also quite popular in the continuous DPP  literature \citep{lavancier2015determinantal, biscio2017contrast,lavancier2021adaptive, poinas2023asymptotic}. 

If the above constrained model is well specified, then we have a true value $\theta_0$ of parameter $\theta$, and the corresponding true value $K_0=K(\theta_0)$ of the kernel. We can also introduce the constrained identified set for the kernel and the identified set for the parameter. They are defined by:
\begin{align*}
	CIS_0&= \{ K=K(\theta): \exists D: K(\theta)= D K_0 D'\}, \\
	\Theta_0&= \{\theta: K(\theta) \in CIS_0\}.
\end{align*}
In other words, $CIS_0$ is the set of kernels of the form $K(\theta)$ that are D-similar with $K_0$ and $\Theta_0$ is the set of such $\theta$'s. Since $CIS_0=IS_0 \cap \{K: K=K(\theta), \theta \in \Theta \}$, this constrained identification set is smaller or equal to the unconstrained identified set $IS_0$.  In other words, the parametric specification can possibly facilitate the identification issue, without necessarily render $K_0$ identifiable. 

At this stage, let us discuss the status of the unconstrained pseudo kernel $K_0^*$. Two cases have to be distinguished since $K_0^*$ may or may not belong to the constrained set $\{K: K=K(\theta), \theta \in \Theta\}$. In the first case, we can write $K_0^*=K(\theta_0^*)$ and define a pseudo true value $\theta_0^*$ of $\theta$. In the other case, this value $K_0^*$ does not belong to the constrained identified set, thus it allows for recovering the distribution of a DPP model, but this distribution does not satisfy the constraints of the parametric model. 

More generally, Proposition 3 can still be used to analyze the identification issue by looking at the $CIS_0$. Examples of constrained DPP model as sparse models, or factor models, are discussed below. 
\begin{exm}
Let us consider the case where $K$ is an equi-covariance matrix:
\begin{equation}
	K(\theta)=\sigma^2(1- \rho) I_d+\sigma^2\rho ee',
	\label{equi}
\end{equation}
where $e$ is the $d-$dimensional vector of $1$'s, $\sigma \geq 0$, the vector of parameters $\theta=(\sigma, \rho)'$ and $d \geq 3.$ The eigenvalues of $K(\theta)$ are $\sigma^2(1-\rho)$ and $\sigma^2[1+(d-1)\rho]$. Thus the constraint on the parameters is that both eigenvalues are strictly between 0 and 1, that is $-\frac{1}{d-1}<\rho<1$. 

On the contrary to the unconstrained case,  both parameters $\rho$ and $\sigma^2$ are identified,  in the sense that no D-similar kernel matrix $K_0=K(\theta_0)$ can be written in this equicovariance form, except $K_0$ itself.
\end{exm}

\begin{exm}
Let us assume that $K(\theta)$ is a Toeplitz matrix, that is, $K(\theta)_{ij}=a_{|i-j|}$ for a sequence $a_0,a_1,....,a_{d-1}$. This can be interpreted as a spatial model, in which the covariance between $X_i$ and $X_j$ depends on the indices $i$ and $j$ through their distance $|i-j|$ only. For this model, the constrained identified set $CIS_0$ is not reduced to singleton, but rather contains $2^{d-1}$ elements, if $a_1,...,a_{d-1}$ are all non-zero. 
\end{exm}

\section{Statistical inference}
 Let us now assume that we observe i.i.d. samples $S_t, t=1,...,T$.  Each $S_t$ can also be equivalently written as a $d$-dimensional binary random vector $(X_{t,1},...,X_{t,d})$.  
 
 The statistical inference for DPP models is challenging due to both the identification issue discussed in Section 2, and the numerical difficulties encountered in computing either the maximum likelihood estimate (MLE), or a method of moment estimate of $K_0$. In this respect, it is important to distinguish the properties of the MLE from the properties of an online stochastic gradient ascent (SGA) estimator introduced as an approximation of the MLE. 
 
 The aim of this section is to explain how the closed form expression of the pseudo kernel $K_0^*$ in terms of the lower-dimensional joint probabilities can be used to facilitate the inference, either for a direct use, or as the starting point of stochastic learning algorithms.  
 
  \subsection{Estimation of a non-constrained DPP model}
 In an unconstrained DPP model, the closed form formulas of Proposition 3, providing the elements of $K_0^*$ in terms of the true distribution, can be used to derive consistent and asymptotically normal estimators $\hat{K}^*_T$ of $K_0^*$. Let us denote by $\hat{K}^*_{ij}, i,j=1,...,d$,  these elements\footnote{For expository purpose, we omit the index $T$ when writing the elements of $\widehat{K^*_T}$. }. They are obtained by simply replacing  the theoretical probabilities and covariances on the right hand side of equations \eqref{identifymargin},\eqref{identifypairwise} and \eqref{identifythreewise} with their sample counterparts: 
\begin{align}
 \widehat{K_{ii}^*}&:= \frac{1}{T} \sum_{t=1}^T X_{t,i}, \label{samplemargin} \\
  \widehat{| K_{ij}^* | }&:= \sqrt{\widehat{-Cov}(X_i, X_j)}. \label{samplepairwise}\\
 \label{samplethreewise}  \widehat{K_{ij}^*}&:= 
\begin{cases}
   \widehat{| K_{ij}^* | },& \text{ if }  \widehat{sgn(K^*_{ij})}=1, \\
 - \widehat{| K_{ij}^* | }, & \text{ if }  \widehat{sgn(K^*_{ij})}=-1,
 \end{cases}  \\
\text{ where } \widehat{sgn(K^*_{ij})}&:=  sgn\Big(\widehat{ \mathbb{P}}[X_1=1, X_i=1, X_j=1]-\widehat{K_{11}^* } \widehat{K^* _{ii}} \widehat{K^* _{jj}}+  \widehat{K_{11}^* }\widehat{|K_{ij}^* |}^2\nonumber \\ 
&\qquad+\widehat{ K^* _{ii}} \widehat{|K^*_{1j}| }^2+ \widehat{K_{jj}} \widehat{|K_{1i}^*| } ^2\Big). \label{orderthree}
 \end{align}
 Due to these closed form expressions, the estimator above has standard asymptotic behaviors, since the challenging identification issue has been solved by focusing on the pseudo true kernel $K^*_0$ instead of the true kernel $K_0$, which is anyway observationally indistinguishable from $K_0$. 

First, by the (strong) law of large numbers and central limit theorem, we have:
 \begin{prop}
 \label{consistency}
 	 Under Assumption 2, the estimator $\widehat{K_T^*}$ defined by equations \eqref{samplemargin},  \eqref{samplepairwise} and \eqref{samplethreewise} converges almost surely to $K_0^*$ as $T$ increases to infinity:
 	 	$$
   \widehat{K^*_{T}} \longrightarrow  K^*_{0}. 
 	$$
 	 In particular, almost surely, for large $T$,  
 	 \begin{itemize}
 	 \item the estimate $\widehat{| K^*_{ij} |}$ in \eqref{samplepairwise} is well defined, i.e., $\widehat{-Cov}( X_i X_j)$ is nonnegative. 
 	 \item $sgn(\widehat{K^*_{ij}})=sgn(K^*_{0,ij})$, that is, we can recover the correct signs of all the elements of $K_0^*$.
 	 \end{itemize}

 	 \end{prop}
 	 
 	 As a comparison,  the consistency result of most learning algorithms in the literature, such as the moment-based learning algorithm of \cite{urschel2017learning},  is not known\footnote{See, e.g.,  \cite{hu2023asymptotic} for the proof of the almost sure convergence of the maximum likelihood estimator of $\Sigma$ parameter of the DPP. }. 
 	 
 	  The estimator $\widehat{K_T^*}$ of $K_0^*$ is symmetric by definition, but does not necessarily satisfy the constraint of all its eigenvalues lying strictly between 0 and 1.   
 	  However, by the law of large numbers,  it is straightforward to show that, as $T$ increases to infinity, the estimator converges almost surely to its pseudo true value, which satisfies these constraints.  In other words, for large  $T$,  the estimator is expected to also satisfy these  constraints.

The expressions of $\widehat{K_{ij}^*}$ are simple functions of the sample cross moments of the observations, 
and  these functions are continuously differentiable in a neighborhood of the pseudo true value $K_0^*$. By Proposition 4, for large $T$, $\hat{K}_T^*$ is also close to $K_0^*$.  
As a consequence,  by the central limit theorem and the delta method, we deduce that:

 \begin{prop}
 \label{normality}
 As $T$ increases to infinity, the distribution of $\sqrt{T} (\widehat{K^*_T}-K_0^*)$ converges to a multivariate normal distribution.
 \end{prop}
 	 
Propositions \ref{consistency} and \ref{normality} differ with previous consistency and normality results of the MLE \citep{brunel2017maximum, hu2023asymptotic}, in that both of these two papers consider maximum likelihood estimator without the identification Assumption 2.  As a consequence,  the likelihood function they propose to maximize could have $2^{d-1}$ global maxima,  and hence the MLE $\hat{K}_{MLE}$ is not assured to converge pointwise to ${K^*_0}$,  nor to ${K_0}$.  Indeed, only $D\hat{K}_{MLE}D$ converges to ${K^*_0}$,  for a certain,  unknown diagonal matrix $D$ with $\pm 1$ on the diagonal.  This makes it difficult to use their asymptotic results to conduct inference. \footnote{For instance,  \cite{hu2023asymptotic}, eq.(3.16)-eq.(3.17) define the MLE ``which is the closest to the given true value ($K_0$)",  to derive the first-order expansion and the asymptotic normality.  However,  this Hu-Shi MLE will depend on the unknown true value,  and then cannot be an estimator. } As a comparison, the asymptotic results derived in Propositions  \ref{consistency} and \ref{normality} do not involve any unknown diagonal matrix,  and can be used to derive the asymptotic behavior of the identified set,  as well as of any identifiable functions of $K^*_0$ (resp. of $K_0$). 
 	 
 $i)$	 Let us first consider the identified set. We have: 
 	 $$
 	 IS_0=\{K: K=D K_0 D, \quad D \text{ varying } \}= \{K: K=D K_0^* D, \quad D \text{ varying} \}.
 	 $$
 	 Its estimator is:
 	 $$
 	 \widehat{IS}_T=\{ \widehat{K(D)}_T: \widehat{K(D)}_T= D \widehat{K_T^*} D,   \quad D \text{ varying}  \}.
 	 $$
 	 
 	 This is a finite set with $2^{d-1}$ elements, which is consistent and asymptotically normally distributed, in the sense that the $2^{d-1}$ dimensional random vector with components $\widehat{K(D)}_T$, $D$ ordered, is multivariate normal. Note that this multivariate distribution (on the space of symmetric matrices) is degenerate since all $\widehat{K(D)}_T$ depends on the single matrix factor $\widehat{K^*_{T}}$.
 	 
 	 $ii)$ These asymptotic results can also be used to derive the consistency and asymptotic normality of any identifiable function of $K_0$, that is a function $\beta_0=\beta(K_0)$, where the function $\beta$ is invariant by D-similarity, that is such that $\beta(K)=\beta(K_0)$ for any $K \in IS_0$. 
 	 
 	 The asymptotic behaviors are more difficult to analyze, when the transformation $\beta(\cdot)$ is not continuously differentiable. This arises,  for instance,  when we look for the set of modes of the DPP distribution, that is when the DPP model is a building block in image analysis by multi-modal generative model.

 \paragraph{Choice of the first component.}  The estimator proposed above depends on which variable is indexed as $X_1$.  Depending on this choice,  we would get up to $d$ pseudo true values $K_0^*$,  if all elements of $K_0$ are nonzero.  Indeed,  if the second row of $K_0$ also contains only nonzero elements, then we could have inverted the first and second rows and columns of $K_0$.  This is equivalent to an estimator of $K^{**}_0$, which is D-similar to both $K_0$ and $K^*_0$, and whose second row is positive. Similarly, this estimator of $K^{**}_0$ is based on the sample principal minors of order 1, 2, as well as the principal minors of the form, say, $\det K_{\{2, i, j\}}$, where $i<j$ and $i \neq 2, j \neq 2$.  
 As a consequence,  the new estimator $\hat{K}^{**}_T$ will converge to another pseudo true kernel $K_0^{**}$. 
 How are these estimators related?  We have the next proposition. 
 \begin{prop}
 For $T$ large enough, the different estimators obtained through different indexation of the binary variables only differ by a D-similarity.  That is,  there exists a diagonal matrix $D$ with $\pm 1$ diagonal terms such that $\widehat{K}^*_T =D \widehat{K}^{**}_T D$.  
 \end{prop}
 This is simply due to the fact that $i)$ two estimators have the same diagonal terms, as well as the same magnitude for off-diagonal terms; $ii)$ for large $T$, each estimator also has exactly the same signs for its off-diagonal terms as its associated pseudo true value, such as $K_0^*$ and $K_0^{**}$; $iii)$ the different pseudo true values such as $K_0^*$ and $K_0^{**}$ are by definition D-similar.  
 
 As a consequence of this proposition, different estimators obtained for different choice of indexation are observationally equivalent for large $T$.  Thus the choice of indexation has asymptotically no impact on the estimate of the identified set and of the distribution of the random set $S$.

 \subsection{Estimation of a constrained DPP model}
 As explained in Section 4,  for a constrained parametric model,  we have to distinguish two cases, depending on whether for a certain diagonal matrix $D$ with $\pm 1$ on the diagonal,  $DK(\theta)D$ is still of the form $K(\theta')$ for some $\theta'$.

 \paragraph{Case 1. } Let us first consider a well-specified constrained model in which $\theta_0$ and $K(\theta_0)=K_0$ are identifiable. In other words, this true value is such that its first row is positive, and then $K_0^*=K_0$. By applying the closed form formulas, the unconstrained estimator $\widehat{K^*_T}$ will now converge almost surely to $K_0=K(\theta_0)$,  at rate $\sqrt{T}$ and be asymptotically normal. 
 
 Then this estimator can be used to estimate $\theta_0$ and the constrained kernel as follows:
 
 \begin{itemize}
 	\item Step 1: Define the moment estimator of $\theta_0$ as: 
 	$$
 	\hat{\theta}_T:=\arg \min_\theta [vech(\widehat{K^*_T})-vech(\widehat{K(\theta)})]'\Omega [vech(\widehat{K^*_T})-vech(\widehat{K(\theta)})],
 	$$
 	where $vech(\cdot)$ is the half-vectorization operator of a symmetric matrix and $\Omega$ is a $(\frac{d(d+1)}{2},\frac{d(d+1)}{2})$ SPD weighting matrix. 
 	\item Step 2: An estimator of the constrained kernel is $\widehat{\widehat{K_T}}=K(\hat{\theta}_T)$. 
 \end{itemize}
The asymptotic behaviors of these estimators (i.e. convergence and asymptotic normality) are easily deduced, if function $K(\cdot)$ is continuous and differentiable,  by applying the delta method. 
 
  \paragraph{Case 2.} Let us now consider the case where $K_0^*$ is not of the form $K(\theta_0^*)$ for a certain $\theta_0^*$.  Then in the procedure above,  after computing the unconstrained estimator $\widehat{K^*_T}$,  the minimization in step 1 should be replaced by:
  	$$
 	\hat{\theta}_T=\arg \min_{\theta} \min_D [vech(D\widehat{K^*_T}D)-vech({K(\theta)})]'\Omega [vech(D\widehat{K^*_T}D)-vech({K(\theta)})].
 	$$ 
This is a joint optimization with respect to both matrix $D$ and parameter $\theta$.  Finally,  Step 2 remains unchanged.

  \section{Comparison with the literature}
  \subsection{An alternative estimation approach}
 \cite{urschel2017learning} introduced an alternative estimation approach presented as ``the first provable guarantees for learning the parameters of a DPP". The approach is based on the assumption that
      the graph induced by $K_0$ is irreducible [\cite{rising2013advances,rising2015efficient, brunel2017rates}], which means that, for any pair $i ,j$,  there exist different vertex $k_1,...,k_q$ such that:
      $$K_{0, i, k_1} \neq 0, K_{0, k_1, k_2} \neq 0,...K_{0, k_q, j} \neq 0.$$
      Then they considered  the cycle sparsity of the graph, that is the smallest integer $\ell_0$ such that the cycles of length at most $\ell_0$ yield a basis for the cycle space of the graph.

  For instance,  if $d=4$, and $K_{12}, K_{23}, K_{34}, K_{41}$ are nonzero, while $K_{13}, K_{24}$ are zero,  then the DPP model is irreducible.  Such a model does not satisfy Assumption 1.  In this section, we explain why the estimation of the DPP under a special irreducible case is based on the same principle.
  
 \subsection{Estimation under a special irreducibility assumption}
Let us now replace Assumption 1 by a weaker assumption:
         \begin{Ass}
 	At least one row of the true kernel $K_0$ contains nonzero elements only.  Without loss of generality, we assume that the first row of $K_0$ contains nonzero elements only. 
 \end{Ass}
Under Assumption 3, we can still make Assumption 2. 
However,  because $K_{0,ij}$ can now equal zero, or equivalently $K^*_{0,ij}$ can equal zero for any $D-$similar kernel $K^*_0$, Proposition 3 has to be modified, along with its estimator, to account for a new definition of $sgn(K_{0,ij})$ that can now be either $+1$, $-1$, or zero by convention, if $K_{0,ij}=0$. More precisely, since in estimator \eqref{samplepairwise}, the term below the square root, i.e.  $-\widehat{Cov}(X_i, X_j)$, converges to zero and is asymptotically normal, for large $T$, there is roughly a 50 \% chance that $-\widehat{Cov}(X_i, X_j)$ is negative. As a consequence, estimator \eqref{samplepairwise} should be replaced by:
 \begin{equation}
 	\label{newamplitude}
   \widehat{| K_{ij}^* | }:= \sqrt{\max[0,\widehat{-Cov}(X_i, X_j)]}. 
\end{equation}
 This estimator is to be compared with the estimator of \cite{urschel2017learning} for $|K_{ij}^* |$, which assigns value zero to $| K_{ij}^* |$ so long as it is smaller than a \textit{predetermined} threshold $\alpha$. However, they do not explain how the value of $\alpha$ should be selected in practice, and in particular, how it depends on $d$ and/or $T$. 
 
 Clearly, the new estimator $|\widehat{K_{ij}^*} |$ given in \eqref{newamplitude} still converges to $|K_{0,ij}^*|$, for any $i, j=1,...,d$. However, the asymptotic distribution of this estimator is no longer normal, if $K_{0,ij}^*=0$. Indeed,  in this case $\sqrt{T} \max[0,\widehat{-Cov}(X_i, X_j)]$ converges to a mixture distribution with two components. The first discrete component, with weight  $\frac{1}{2}$, is the point mass at zero; whereas the second component is the distribution of $|Z|$, where $Z$ follows a normal distribution. 
 
Similarly, the estimator of $sgn(K^*_{ij})$ should be modified as follows, depending on the value of $  \widehat{| K_{ij}^* | }$ found in eq. \eqref{newamplitude}:
 \begin{itemize}
 \item if $\widehat{| K_{ij}^* | }=0$, then we define $\widehat{sgn(K^*_{ij})}=0$ and $\widehat{K_{ij}^* }=0$. 
 \item if $\widehat{| K_{ij}^* | }>0$, then $\widehat{sgn(K^*_{ij})}$ and $\widehat{K_{ij}^* }$ continues to be defined by eq.\eqref{samplethreewise}.
 \end{itemize}
It is easily checked that we still have $\widehat{K^*_T} \rightarrow K^*_0$ almost surely. 
  
  \subsection{Large deviation}
  A large part of \cite{urschel2017learning}'s analysis of their estimator is in deriving an upper bound for a distance between the estimator and the true identified set. More precisely, let us denote the pseudo-distance:
  $$
  \rho(K, K');=\min_D |DKD-K'|_\infty,
  $$
  between two kernels $K, K'$, where $|K|_\infty=\max_{ij}|K_{ij}|$. We are interested in an upper bound for $ \rho(K_0, \widehat{K^*_T})$, where:
  \begin{equation}
  	\label{pseudodistance}
  \rho(K_0, \widehat{K^*_T}) \leq |K_0^*- \widehat{K^*_T}|_\infty,
\end{equation}
by definition. The following proposition relies on the fact that the estimator introduced in Proposition 3 is a continuously differentiable function of the sample moments to give a bound for the probability that this distance $|K_0^*- \widehat{K^*_T}|_\infty,$ exceeds a certain threshold. 

 \begin{prop}
 	\label{largedeviation}
 	For any given $\epsilon>0$, there exists a constant $\eta(\pi_0, \epsilon)>0$ such that for any $T$, we have: 
 	\begin{equation}
 		\label{bound}
\mathbb{P}[ \rho(K_0, \widehat{K^*_T}) > \epsilon] \leq 2\sum_{h=1}^{d^2-d+1}  e^{-[\eta(\pi_0, \epsilon) \ln\frac{\eta(\pi_0, \epsilon)}{\pi_{0,h}}+(1-\eta(\pi_0, \epsilon)) \ln\frac{1-\eta(\pi_0, \epsilon)}{1-\pi_{0,h}}]T}.
\end{equation}
\end{prop}
\begin{proof}
	See Appendix. 
\end{proof}

In machine learning, this upper bound is sometimes used for determining a minimal size $T^*$ of the sample to reach this $\epsilon-$ performance of the estimator [see, e.g., \cite{urschel2017learning}, Thm 1]. It is obtained by inverting the upper bound with respect to $T$ and is often called sample complexity. Note that it depends on the dimension $d$ through the term $d^2-d+1$, but also by means of $\eta(\pi_0, \epsilon)$. 

\begin{rem}
Proposition 7 is based on a large deviation principle, which says that  inequality \eqref{bound} is sharp,  in the sense that we can not replace the coefficient $[\eta(\pi_0, \epsilon) \ln\frac{\eta(\pi_0, \epsilon)}{\pi_{0,h}}+(1-\eta(\pi_0, \epsilon)) \ln\frac{1-\eta(\pi_0, \epsilon)}{1-\pi_{0,h}}]$ in front of $T$ by a smaller one. For instance, in the Appendix, we explain that if we apply Hoeffding's inequality as in Theorem 1 of \cite{urschel2017learning}, we get a larger upper bound.  
\end{rem}

  \section{Concluding remarks}
 It is often believed that the DPP models, characterized by the correlation kernels, are difficult to estimate. In particular, the estimators proposed in the literature, including the moment based algorithm of \citep{urschel2017learning}, can be costly in terms of computation time, and their statistical properties are not well known. This paper has introduced a simple, closed form alternative estimator of the correlation kernel that converges pointwise to a pseudo-true kernel in the identified set, and is asymptotically normal. 

Our approach can be easily extended to other types of DPP model as the Markov DPP model that appear in recurrent stochastic neural network for multivariate binary time series \citep{gourierouxmarkovdpp}.
 \bibliographystyle{apalike}
\bibliography{lib}

 \appendices
\appendix 
\section{Proof of Proposition 7}
Let us denote by $\pi$ the vector which stacks the $d$ marginal probabilities $\mathbb{P}[X_i=1]$,  the $\frac{d(d-1)}{2}$ pairwise probabilities $\mathbb{P}[X_i=X_j=1],  i <j $, as well as the $\frac{(d-1)(d-2)}{2}$ threewise probabilities $\mathbb{P}[X_1=X_i=X_j=1], 1<i<j \leq d$\footnote{In other words, $\pi$ is of dimension $d+\frac{d(d-1)}{2}+\frac{(d-1)(d-2)}{2}=d^2-d+1$.} 
We also denote by $\pi_0$ its true value and $\widehat{\pi_T}$ its sample counterpart. Then we have: 
\begin{equation}
	\label{afunction}
	K^*_0=k(\pi_0), \qquad \widehat{K^*}_T=k(\widehat{\pi_T}),
\end{equation}
where $k(\cdot)$ is the function given in Proposition 3,  which expresses $K^*_{ij}$ as a function of these probabilities.  Clearly,  this function $k(\cdot)$ is continuous in a neighborhood of the true value $\pi_0$. 

We can now use either the large deviation principle  or the law of iterated logarithm  applied to the components of $\widehat{\pi_T}$ to get bounds.   Indeed,  by equations \eqref{pseudodistance}-\eqref{afunction}, we have:
$$
| K^*_0-\widehat{K^*}_T|_\infty=  |k(\pi_0) -k(\widehat{\pi_T})|.
$$
By the continuity of function  $k(\cdot)$,  so long as $\widehat{\pi_T}$ is close enough to $\pi_0$,  the value of the function $k(\widehat{\pi_T})$ will be close enough to $k(\pi_0)$.  More precisely, for any $\epsilon>0$,  there exists a positive threshold $\eta(\pi_0, \epsilon)$ such that:
$$
|\pi_0-\widehat{\pi_T}  |_\infty \leq \eta(\pi_0, \epsilon) \Longrightarrow  |k(\pi_0) -k(\widehat{\pi_T})|_\infty\leq \epsilon,
$$
or equivalently:
$$
|k(\pi_0) -k(\widehat{\pi_T})|_\infty >  \epsilon \Longrightarrow  |\pi_0-\widehat{\pi_T}  |_\infty > \eta(\pi_0, \epsilon). 
$$
Thus we get:
\begin{align*}
	\mathbb{P}[ \rho(K_0, \widehat{K^*_T}) > \epsilon]&  \leq \mathbb{P}[|K_0^*- \widehat{K^*_T}|_\infty>\epsilon] \\
	& \leq \mathbb{P}[ |\pi_0-\widehat{\pi_T}  |_\infty > \eta(\pi_0, \epsilon)] \\
	& \leq \sum_{h=1}^{d^2-d+1}   \mathbb{P}[ |\pi_{0,h}-\widehat{\pi_{T,h}}  | > \eta(\pi_0, \epsilon)],
\end{align*}
where $\pi_{0,h}$ and $\widehat{\pi_{T,h}}$ are the $h$-th components of vector $\pi_{0}$ and $\widehat{\pi_{T}}$,  respectively. 

Let us now bound each of these $d^2-d+1$ probabilities.  Let us now first use the large deviation.  Since $\widehat{\pi_{T,h}}$ is a sample average of Bernoulli variables,  applying the large deviation principle leads to \citep[eq.2]{nagaev2018large}:
\begin{equation}	
	\label{largedeviation}
	\mathbb{P}[|\pi_{0,h}-\widehat{\pi_{T,h}}  |> \eta(\pi_0, \epsilon) ]  \leq 2 e^{-[\eta(\pi_0, \epsilon) \ln\frac{\eta(\pi_0, \epsilon)}{\pi_{0,h}}+(1-\eta(\pi_0, \epsilon)) \ln\frac{1-\eta(\pi_0, \epsilon)}{1-\pi_{0,h}}]T}.
\end{equation}
where, by convention, if $\pi_{0,h}$ is equal to 0 or 1, then the right hand side of eq.\eqref{largedeviation} is equal to 0.  

Thus $$
\mathbb{P}[ \rho(K_0, \widehat{K^*_T}) > \epsilon] \leq 2\sum_{h=1}^{d^2-d+1}  e^{-[\eta(\pi_0, \epsilon) \ln\frac{\eta(\pi_0, \epsilon)}{\pi_{0,h}}+(1-\eta(\pi_0, \epsilon)) \ln\frac{1-\eta(\pi_0, \epsilon)}{1-\pi_{0,h}}]T}.
$$
\begin{rem}

If instead we apply Hoeffding's inequality as in Theorem 1 of \cite{urschel2017learning}, we would get:
\begin{equation}	
	\label{hoeffding}
	\mathbb{P}[|\pi_{0,h}-\widehat{\pi_{T,h}}  |> \eta(\pi_0, \epsilon) ]  \leq 2 e^{-2\eta(\pi_0, \epsilon)^2 T}.
\end{equation}
For $\eta(\pi_0, \epsilon)$ small enough,  we have: $$[\eta(\pi_0, \epsilon) \ln\frac{\eta(\pi_0, \epsilon)}{\pi_{0,h}}+(1-\eta(\pi_0, \epsilon)) \ln\frac{1-\eta(\pi_0, \epsilon)}{1-\pi_{0,h}}]> \eta(\pi_0, \epsilon)^2.$$
 Hence inequality \eqref{largedeviation} is sharper than \eqref{hoeffding}. 
\end{rem}
\begin{rem}
	Similarly,  
	we can also apply the law of iterated logarithm, which says that for each $h$,  
	\begin{equation}
		\label{iteratedlog}
		\limsup_{T \to \infty}  \frac{|\pi_{0,h}-\widehat{\pi_{T,h}}  |}{\sqrt{\pi_{0,h}(1-\pi_{0,h})} \sqrt{2 \ln \ln T}} \sqrt{T} =1, \quad a.s.,
	\end{equation}
	if $\pi_{0,h}$ is different from 0 and 1, and $\limsup_{T \to \infty}  {|\pi_{0,h}-\widehat{\pi_{T,h}}  |}=0$ if $\pi_{0,h}$ is equal to 0 or 1. 
	
	That is, instead of controlling the probability $\mathbb{P}[ \rho(K_0, \widehat{K^*_T}) > \epsilon]$,  eq.\eqref{iteratedlog} provides an almost sure bound for the absolute error $|\pi_{0,h}-\widehat{\pi_{T,h}}  |$. 
\end{rem}


\end{document}